\documentclass{article}
\usepackage{spconf}
\usepackage{amsmath}
\usepackage{graphicx}
\usepackage{amssymb}
\usepackage{amsthm}
\usepackage{cite}
\usepackage{algorithm}
\usepackage{algpseudocode}
\usepackage{url}
\usepackage{booktabs}
\usepackage{multirow}
\newcommand{\E}{\mathbb{E}}

\newcommand{\Prob}{\mathbb{P}}
\newcommand{\R}{\mathbb{R}}
\newtheorem{theorem}{Theorem}
\newtheorem{lemma}[theorem]{Lemma}

\newtheorem{definition}[theorem]{Definition}
\newcommand{\norm}[1]{\left\lVert #1 \right\rVert}
\newcommand{\abs}[1]{\left| #1 \right|}

\newcommand{\Bern}{\text{Bernoulli}}

\usepackage{mathtools}
\usepackage{bm}

\theoremstyle{remark}
\newtheorem{remark}{Remark}

\newcommand{\I}{\mathbf{I}}
\newcommand{\e}{\mathbf{e}}

\usepackage{hyperref}
\title{Misclassification Rate and Privacy-Utility Trade-offs in Graph Convolutional Networks via Subsampling Stability}

\name{Yexin Zhang$^{1,2}$, Zhongtian Ma$^{1,2}$, Qiaosheng Zhang$^{*2,3}$, Zhen Wang$^{*1}$\thanks{This research was supported by the National Natural Science Foundation of China (No.U22B2036), the Science and Technology Commission of Shanghai Municipality (STCSM) under Grant No.25GA3200200, the Technological Innovation Team of Shaanxi Province (No. 2025RS-CXTD-009), the International Cooperation Project of Shaanxi Province (No. 2025GH-YBXM-017), and the Shanghai Artificial Intelligence Laboratory. *Corresponding author: zhangqiaosheng@pjlab.org.cn, w-zhen@nwpu.edu.cn.}\thanks{This arXiv version corrects the qualitative interpretation of the subsampling probability $p_s$ in the privacy--utility discussion of the ICASSP 2026 version, without changing the definition of $p_s$, Algorithm~1, or the theoretical results.}}

\address{
    $^{1}$School of Cybersecurity, Northwestern Polytechnical University, Xi'an, China \\
    $^{2}$Shanghai Artificial Intelligence Laboratory, Shanghai, China \\
    $^{3}$Shanghai Innovation Institute, Shanghai, China
}

\begin{document}
%
\maketitle
\begin{abstract}

We study differential privacy (DP) in Graph Convolutional Networks (GCNs) through the framework of \textit{subsampling stability}. We derive upper bounds on the misclassification rate that depend explicitly on the subsampling probability $p_s$. Furthermore, we characterize the \textit{privacy--utility trade-off} by identifying feasible ranges of $p_s$; if $p_s$ is too large, the stability-based privacy condition becomes difficult to satisfy, yielding vacuous guarantees, whereas if it is too small, accuracy deteriorates. Our results provide the first rigorous theoretical framework for understanding subsampling stability in GCNs under DP.

\end{abstract}

\section{Introduction}

Graph signal processing has become instrumental in many domains, such as network science, recommendation systems, and cybersecurity~\cite{scarselli2008graph,fan2019graph}. Graph Neural Networks (GNNs), and in particular Graph Convolutional Networks (GCNs), have emerged as powerful tools for learning from such data, leveraging iterative aggregation of neighborhood information to construct hierarchical node representations~\cite{kipf2016semi}. By combining graph filters with pointwise nonlinearities, GCNs effectively capture multi-hop dependencies through localized spectral convolutions, achieving strong performance in tasks such as node classification~\cite{tang2022graph} and community detection.

However, the growing use of GCNs raises significant privacy concerns, as graph data often contains sensitive interactions between entities~\cite{karwa2011private, he2021stealing,olatunji2021membership}. Differential privacy (DP) provides a rigorous framework that ensures an algorithm’s output remains nearly unchanged when any single data point is modified~\cite{dwork2014algorithmic}.
DP has become a popular research direction for privacy protection in graph data, particularly in the context of GCNs~\cite{sajadmanesh2024progap}. Existing studies span several directions, including perturbing node features or edges to protect sensitive inputs~\cite{daigavane2021node}, applying DP-SGD style training to safeguard gradients during optimization, and developing methods for private graph data publishing and querying~\cite{abadi2016deep, ayle2023training,sajadmanesh2023gap}.

Despite these advances, theoretical investigations of DP in the context of GCNs remain scarce. A rigorous theoretical foundation, however, is essential for understanding the fundamental limits of privacy preservation and guiding principled algorithm design~\cite{mohamed2022differentially}. Such theory has been extensively studied in related settings, including spectral clustering \cite{koskela2025price} and semi-definite programming (SDP) \cite{seif2024differentially}. In this work, we aim to provide such a perspective by analyzing a well-trained GCN under the incorporation of DP guarantees through a \emph{subsampling stability mechanism}~\cite{thakurta2013differentially,testa2024stability}. Our focus is to rigorously characterize the resulting \emph{privacy–utility trade-offs} and to establish misclassification bounds that explicitly depend on the subsampling parameter.

Our contributions are three-fold: 
\begin{enumerate}
\item \textbf{Subsampling stability for DP in GCNs.} To the best of our knowledge, we are the first to introduce and analyze a subsampling-stability mechanism to provide differential privacy guarantees for GCNs; 
\item \textbf{Misclassification rate under subsampling.} We derive an explicit upper bound on the misclassification rate under this mechanism, where the bound depends directly on the subsampling probability;
\item \textbf{Characterization of privacy--utility trade-offs.} We establish precise analytical expressions that quantify the resulting privacy--utility trade-off by deriving feasible ranges of the subsampling probability $p_s$, showing in particular that $p_s$ cannot be too small or too large to simultaneously guarantee privacy and utility.
\end{enumerate}

\section{Preliminaries and problem setup}
\label{sec:pagestyle}
\subsection{Graph Convolutional Networks}

Given an input graph $G = (V, E)$ with $n = |V|$ nodes, the adjacency matrix of the graph is denoted by $\mathbf{A} \in \{0,1\}^{n \times n}$. We denote the degree matrix \(\mathbf{D}=\mathrm{diag}(d_1,\ldots,d_n)\)  with \(d_i=\sum_{j=1}^n A_{ij}\), and the (combinatorial) graph Laplacian as \(\mathbf{L}=\mathbf{D}-\mathbf{A}\). When only first-order neighborhood information is retained (i.e., $k = 1$), a single-layer GCN can be expressed as,
\begin{equation}
\mathbf{y} = \sigma \left( h_0 \mathbf{I} + h_1 \mathbf{L} \right) \mathbf{x},
\label{eq.gcn}
\end{equation}
where $\mathbf{x} \in \mathbb{R}^n$ is the input node feature vector, $\mathbf{I}$ is the identity matrix, $h_0$, $h_1 \in \mathbb{R}$ are filter coefficients, and $\sigma(\cdot)$ is the nonlinear activation function. 

\subsection{Differential privacy}
To protect structural privacy in graph-valued data, we adopt the framework of DP under edge-level perturbations. The core idea, rooted in the DP principle of indistinguishability between neighboring datasets, is to ensure that the distribution of the estimator’s output remains stable under small structural modifications of the input graph.

In the context of graph learning, we use the notion of \emph{neighboring graphs}\cite{karwa2011private}, which serves as the adjacency relation for DP in the graphs.

\begin{definition}[Neighboring Graphs]\label{def:neighboring_graphs}
Two graphs $G = (V, E)$ and $G' = (V', E')$ are said to be \emph{neighboring}, denoted by $G \sim G'$, if they share the same vertex set ($V = V'$) and their edge sets differ in exactly one edge. That is,
$
|E \mathbin{\triangle} E'| = 1,
$
where $\triangle$ denotes the symmetric difference.
\end{definition}

Using this adjacency relation, we define $(\epsilon, \delta)$-edge differential privacy following \cite{dwork2014algorithmic} (see also \cite{mohamed2022differentially}), which bounds the distinguishability between the output distributions on neighboring graphs.

\begin{definition}[($\epsilon, \delta$)-Edge Differential Privacy]
Let $\epsilon > 0$ and $\delta \in (0,1)$. A randomized estimator $\hat{\mathbf{y}}$ satisfies \emph{$(\epsilon, \delta)$-edge DP} if for every pair of neighboring graphs $G \sim G'$ and every measurable set of outputs $\Phi \subseteq \mathrm{Range}(\hat{\mathbf{y}})$,
\begin{align}\label{def:edge_dp}
    \Pr\big( \hat{\mathbf{y}}(G) \in \Phi \big) \leq e^{\epsilon} \cdot \Pr\big( \hat{\mathbf{y}}(G') \in \Phi \big) + \delta,
\end{align}
where the probability is taken over the internal randomness of $\hat{\mathbf{y}}$. When $\delta = 0$, the mechanism satisfies \emph{pure $\epsilon$-edge DP}; otherwise, it satisfies \emph{approximate $(\epsilon, \delta)$-edge DP}.
\end{definition}


\subsection{Node classification}
Node classification is among the most common tasks for GCNs, with the goal of inferring unknown node labels from the observed graph topology and associated node features~\cite{luan2023graph}. Concretely, given a graph represented by an adjacency matrix $\mathbf{A}$, ground-truth labels $\mathbf{y}  \in \{\pm 1\}^n$ and a node feature vector $\mathbf{x}$, an algorithm $\hat{\mathbf{y}}(\mathbf{A}, \mathbf{x})$ produces a predicted label vector $\hat{\mathbf{y}}$. We measure the prediction error by using \textit{misclassification rate} with the Hamming distance, defined as
$R(\mathbf{y}, \hat{\mathbf{y}}):= \frac{\mathrm {Ham}(\mathbf{y}, \hat{\mathbf{y}})}{n}$ and 
$\mathrm {Ham}(\mathbf{y}, \hat{\mathbf{y}}) = \left| \left\{ i \in [n] : y_i \neq \hat{y}_i \right\} \right|$.




\section{Subsampling stability Mechanism}
Given an input graph $G = (V, E)$ with $n = |V|$ nodes, we generate $m$ independent subgraphs $\{G_1, \dots, G_m\}$ via edge-level subsampling, where each edge is retained independently with probability $p_s \in (0,1]$. For each subsampled graph $G_\ell$, a base GCN classifier produces a binary node labeling $\hat{\mathbf{y}}(G_\ell) \in \{\pm1\}^n$. The stability of this classifier under subsampling is critical to the overall mechanism. We aggregate the $m$ label vectors $\{\hat{y}(G_l)\}_{l=1}^m$ via majority voting, denoted by $g(\cdot)$, to produce a consensus estimate $\bar{\mathbf{y}}$. To evaluate whether this estimate is sufficiently stable for release under differential privacy, we compute a \textit{stability score} of the results of majority voting: 
$
 \hat{d} = \frac{c_1 - c_2}{4m\cdot p_s} - 1,
$
where $c_1$ and $c_2$ represent the frequencies of the first and second most common label vectors among the $m$ outputs.

To achieve $(\epsilon, \delta)$-edge DP, we adopt the \emph{Propose-Test-Release} (PTR) framework~\cite{dwork2014algorithmic}. We perturb $\hat{d}$ with Laplace noise of scale $1/\epsilon$, yielding $\tilde{d} = \hat{d} + \mathrm{Lap}(0, 1/\epsilon)$. If $\tilde{d} > \log(1/\delta)/\epsilon$, we get the final output $\bar{\mathbf{y}}$; otherwise, the algorithm returns $\perp$, indicating insufficient stability. The parameters $p_s$ and $m$ are chosen based on the derived upper bound of the misclassification rate and the requirements of the DP mechanism, aiming to achieve an optimal trade-off between privacy and utility.
\begin{algorithm}[ht]
    \caption{$A_{\text{samp}}$GCN}
    \label{alg:subsample_stability}
    \begin{algorithmic}[1]
        \State \textbf{Input:} Graph $G = (V, E)$ with $ \mathbf{x}$, privacy budget $(\epsilon, \delta)$, graph structural parameters.
        \State \textbf{Output:} Private labeling vector $\hat{\mathbf{y}}_{\mathrm{final}}$ or $\perp$.
        
        \State Select $p_s \in ( p_s^* , \, \frac{\epsilon}{32 \ln( 1/\sigma )  } ) $ and $m = \frac{\log(n/\delta)}{p_s^2}$
        \State Generate $m$ subsampled graphs $\{G_1, \dots, G_m\}$ by retaining each edge independently with probability $p_s$
        \State $\hat{\mathbf{y}}(G_\ell) \gets \text{GCN}(G_\ell)$ for each $\ell = 1,\dots,m$
        \State Aggregate: $\bar{\mathbf{y}} \gets g\big(\hat{\mathbf{y}}(G_1), \dots, \hat{\mathbf{y}}(G_m)\big)$
        \State Compute $\hat{d} \gets \frac{c_1 - c_2}{4m p_s} - 1$
        \State $\tilde{d} \gets \hat{d} + \mathrm{Lap}(0, 1/\epsilon)$
        \If{$\tilde{d} > \log(1/\delta)/\epsilon$}
            \State \textbf{return} $\bar{\mathbf{y}}$
        \Else
            \State \textbf{return} $\perp$ 
        \EndIf
    \end{algorithmic}
\end{algorithm}

We characterize the \emph{subsampling stability} of GCNs by deriving an upper bound on the misclassification rate. To obtain this bound, we first analyze the misclassification rate on a single subsampled graph, and then aggregate the results through majority voting to establish the misclassification rate under the full framework.



\section{Theoretical Results}
We first derive upper bounds on the misclassification rate for both a single subsampled graph and the full $A_{\text{samp}}$GCN (Theorem~\ref{thm:ps_subsampling_stability} and Theorem~\ref{lemma:overlap-rate}). We then provide a quantitative characterization of the privacy–utility trade-off exhibited by $A_{\text{samp}}$GCN (Theorem~\ref{thm:asamp-gcn}).

\begin{theorem}[Misclassification Rate of GCN on a Single Subsampled Graph]\label{thm:ps_subsampling_stability}
Let $\widehat{G}\sim\mathcal D_s(G)$ be obtained by independently retaining each edge with probability $p_s\in(0,1]$.
Consider the one-layer GCN in~\eqref{eq.gcn} with an elementwise activation function $\sigma(\cdot)$ that is $C_\sigma$-Lipschitz in $\ell_2$ norm.
Let $f(G):=\sigma\!\big((h_0\mathbf I+h_1\mathbf L)\mathbf x\big)$, and define
$\mathbf y:=\mathbf y(G)\in\{\pm1\}^n$ and $\hat{\mathbf y}:=\mathbf y(\widehat G)\in\{\pm1\}^n$
by thresholding $f(\cdot)$ at level $\tau$. And define the minimum margin $\gamma_{\min}:=\min_{i\in[n]}|f_i(G)-\tau|$, assuming $\gamma_{\min}>0$.
Assume that $\|\mathbf x\|_2=1$ ($L_2$-normalization). Then, for every $\eta\in(0,1)$, the following inequality holds with probability at least $1-\eta$:
\begin{align}
& R(\mathbf{y}, \hat{\mathbf{y}})
 \le f(p_s, \eta), \\
& \text{where} \, f(p_s, \eta) := \frac{C_\sigma\,|h_1|}{\sqrt{n} \,\,\,\gamma_{\min}}\,
\Big\{(1-p_s)\,\|\mathbf{L}\|_2\notag \\
&\qquad\qquad\qquad\qquad+\sqrt{\,4\,p_s(1-p_s)\,\|\mathbf{L}\|_2\,\log(2n/\eta)}\notag\\
&\qquad\qquad\qquad\qquad+\tfrac{4}{3}\log(2n/\eta)\Big\}.
\label{eq:bound}
\end{align}
\end{theorem}


\begin{theorem}[Misclassification Rate for Subsampling Stability Mechanism]\label{lemma:overlap-rate}
Consider the subsampling stability-based mechanism with  $(\epsilon, \delta)$-edge DP. Let $\mathbf{y}, \bar{\mathbf{y}}$ denote the original results and the estimated labels obtained via mechanism over $m$ subgraphs $\{G_i\}_{i=1}^{m}$. Assume that $\| \mathbf{x} \|_2=1$, then the misclassification rate satisfies:
\[
R(\mathbf{y}, \bar{\mathbf{y}}) \le  \frac{f(p_s,\eta)}{4} + O\left(\frac{1}{\sqrt{m}}\right),
\]
with probability at least $1 - \eta$, where $f(p_s,\eta)$ is the misclassification contribution from the GCN for a single subsampled graph $G_\ell$.
\end{theorem}

\begin{remark}
The proof of Theorem~\ref{lemma:overlap-rate} follows from combining the misclassification bound for a single graph with the majority voting scheme and Hoeffding’s inequality; see Lemma III.6 in~\cite{koskela2025price} for details.
\end{remark}

\begin{theorem}[Privacy-Utility Trade-off]\label{thm:asamp-gcn}
The following guarantees hold for $A_{\text{samp}}GCN$ if the subsampling probability $p_s \in ( p_s^* , \, \frac{\epsilon}{32 \ln( 1/\sigma ) ) } )$: \\
\textbullet \ \textbf{Privacy:} The $A_{\text{samp}}GCN$ method is $(\varepsilon,\delta)$-edge differentially private. \\
\textbullet \ \textbf{Utility:}
$
\Pr[\,A_{\text{samp}}GCN(G)=\mathbf{y}(G)\,]\ge 1-3\delta,
$ \\
where $p_s^*$ is the larger root of the following quadratic equation:
\begin{equation}
\begin{aligned}
        & (1-p_s) \|\mathbf{L}\|_2 + \sqrt{4\,\|\mathbf{L}\|_2\,\log(8n)\,p_s(1-p_s)} \\
        & +  \tfrac{4}{3}\log(8n) - \frac{\gamma_{\min}}{\sqrt{n}\,C_\sigma\,|h_1|} = 0
\end{aligned}
\end{equation}
\end{theorem}
\begin{remark}
The upper bound on the sampling rate \( p_s \) is imposed to ensure compatibility with the DP mechanism and is determined by the prescribed privacy parameters, with Theorem 7.3 of~\cite{dwork2014algorithmic} showing that this bound indeed guarantees DP. The lower bound, derived from Theorem~\ref{thm:ps_subsampling_stability} (by setting \( \eta = 1/4 \) and requiring the misclassification rate to be less than \( 1/n \)), guarantees that, for any single graph, \( \Pr\{\mathrm{Ham}(\mathbf{y}, \hat{\mathbf{y}}) = 0\} > 3/4 \). As shown in Theorem 7.3 of~\cite{dwork2014algorithmic}, this probabilistic guarantee further implies a formal utility bound.

Intuitively, if \( p_s \) is too large, the stability condition becomes difficult to satisfy, yielding vacuous privacy–utility guarantees; conversely, if \( p_s \) is too small, it suppresses the extraction of informative signals, leading to significant degradation in accuracy. Thus, \( p_s \) inherently governs the fundamental privacy–utility trade-off.

\end{remark}

\section{proof of Theorem~3}
We consider a one-layer GCN $f(G) :=\sigma\!\big((h_0\I+h_1\mathbf L)\,\mathbf x\big)$, where $h_0,h_1\in\R$ are the filter parameters, $\mathbf x\in\R^n$ represents node feature vector, and $\sigma$ is the nonlinear activation function. Thus the output of the subsampled graph is $f(\widetilde G):=\sigma\!\big((h_0\I+h_1\widetilde{\mathbf L})\,\mathbf x\big)$. 

The feature-level change is measured by 
$
    d_f(G,\widetilde G):=\norm{f(G)-f(\widetilde G)}_2.
$
For convenience, we write $H(\mathbf L):=(h_0\I+h_1\mathbf L)\mathbf x$, $f(G)=\phi(\mathbf L):=\sigma\!\big(H(\mathbf L)\big)$ and $f(\widetilde G)=\phi(\widetilde {\mathbf L}):=\sigma\!\big(H(\widetilde {\mathbf L})\big)$.

Let \(\mathbf{y}(G),\mathbf{y}(\widetilde{G})\in\{\pm1\}^n\) be estimated vectors obtained by thresholding at level \(\tau\).
We measure the output distance between original graph and subsampled graph with the Hamming distance of labels
\[
\text{Ham}\big(\mathbf{y}(G),\mathbf{y}(\widetilde{G})\big)
= \Big|\big\{\, i\in[n] : y_i(G)\neq y_i(\widetilde{G}) \,\big\}\Big|.
\]

\begin{lemma}[Flip-count via margins]\label{lem:margin}
If $\sigma$ is elementwise and $C_\sigma$-Lipschitz in $\ell_2$ norm, then for any $\mathbf x$ with $\norm{\mathbf x}_2=1$,
\begin{equation}
\mathrm {Ham} \!\big(\mathbf y(G),\mathbf y(\widetilde G)\big)
\le
\frac{\sqrt n\,C_\sigma\,\abs{h_1}}{\gamma_{\min}}\,
\norm{\widetilde{\mathbf L}-\mathbf L}_2.
\end{equation}
\end{lemma}

\begin{proof}
If a flip occurs at node $i$, then 
\begin{align}
\abs{\phi_i(\mathbf L)-\tau}\le \abs{\phi_i(\mathbf L)-\phi_i(\widetilde{\mathbf L})}.
\end{align}
Summing over flipped points and using the standard $\ell_1$–$\ell_2$ inequality yields
\begin{align}
\mathrm {Ham} \!\big(\mathbf y(G),\mathbf y(\widetilde G)\big)&\le \frac{1}{\gamma_{\min}}\norm{\phi(\mathbf L)-\phi(\widetilde{\mathbf L})}_1\\
&\le \frac{\sqrt n}{\gamma_{\min}} \norm{\phi(\mathbf L)-\phi(\widetilde{\mathbf L})}_2.\label{12}
\end{align}
By Lipschitzness of $\sigma$, linearity of $H(\cdot)$ and $\norm{\mathbf x}_2=1$, we have
\begin{align}
    \norm{\phi(\mathbf L)-\phi(\widetilde{\mathbf L})}_2
&\le C_\sigma \norm{H(\mathbf L)-H(\widetilde{\mathbf L})}_2\\
&\le C_\sigma \abs{h_1}\,\norm{\widetilde{\mathbf L}-\mathbf L}_2.\label{14}
\end{align}
Substituting Eq.\eqref{14} into Eq.\eqref{12} completes the proof of Lemma \ref{lem:margin} 
\end{proof}

Then we turn to analyze $\big\|\,\Delta\mathbf{L} \big\|_2 := \| \widehat{\mathbf{L}} - \mathbf{L} \|$. Firstly, we decompose the original Laplacian of graphs $\mathbf{L}=\mathbf{D}-\mathbf{A}$. The adjacency matrix $\mathbf{A}$ and degree matrix $\mathbf{D}$ can be decomposed as follows:
\begin{equation}
    \mathbf{A} = \sum_{i < j} \mathbf{A}_{ij} (\mathbf{e}_i \mathbf{e}_j^{\top} + \mathbf{e}_j \mathbf{e}_i^{\top}),
\end{equation}
\begin{equation}
    \mathbf{D} = \sum_{i < j} \mathbf{A}_{ij} (\mathbf{e}_i \mathbf{e}_i^{\top} + \mathbf{e}_j \mathbf{e}_j^{\top}),
\end{equation}
then we obtain 
$
   \mathbf{L} = \sum_{i < j} \mathbf{A}_{ij} (\mathbf{e}_i - \mathbf{e}_j)(\mathbf{e}_i - \mathbf{e}_j)^{\top},
$
where $\mathbf{e}_i$ and $\mathbf{e}_j$ are the standard basis vectors.

For each edge $(i,j)$, we define a Bernoulli random variable $\mathbf{B}_{ij} \sim \Bern(p_s)$ that determines edge retention. This process yields a perturbed adjacency matrix $\widehat{\mathbf{A}}$, where each entry $\widehat{\mathbf{A}}_{ij} = \mathbf{A}_{ij} \cdot \mathbf{B}_{ij}$. The perturbed degree matrix $\widehat{\mathbf{D}}$ is consequently defined as a diagonal matrix with $\widehat{\mathbf{D}}_{ii} = \sum_j \widehat{\mathbf{A}}_{ij}$. The perturbed adjacency matrix and perturbed degree matrix admits the same decomposition. The subsampled graph Laplacian is given as
$
\widehat{\mathbf{L}} = \sum_{i<j} \mathbf{A}_{ij}\mathbf{B}_{ij}(\mathbf{e}_i - \mathbf{e}_j)(\mathbf{e}_i - \mathbf{e}_j)^\top.
$
Then we have 
$
\Delta\mathbf L = \sum_{i<j}  \mathbf{A}_{ij}(\mathbf{B}_{ij}-1)\,(\e_i-\e_j)(\e_i-\e_j)^\top .
$

Let $\mathbf X_{ij}:=\mathbf{A}_{ij}(\mathbf{B}_{ij}-1)\,{\mathbf u}_{ij}{\mathbf u}_{ij}^\top$ with ${\mathbf u}_{ij}:=\e_i-\e_j$, and define centered terms as
\begin{equation}
\mathbf Y_{ij}:=\mathbf X_{ij}-\E[\mathbf X_{ij}]
= \mathbf{A}_{ij}(\mathbf{B}_{ij}-p_s)\,\mathbf u_{ij}\mathbf u_{ij}^\top.
\end{equation}
Then
\begin{equation}
\Delta\mathbf L
= \sum_{i<j}\E[\mathbf X_{ij}] + \sum_{i<j}\mathbf Y_{ij}
= (p_s-1)\mathbf L + \sum_{i<j}\mathbf Y_{ij}.
\end{equation}

For $\mathbf A_{ij}\in\{0,1\}$, each $\mathbf Y_{ij}$ is self-adjoint, mean-zero and satisfies:
\begin{equation}
     \norm{\mathbf Y_{ij}}_2 \le 2, \quad\sum_{i<j}\E[\mathbf Y_{ij}^2]=2p_s(1-p_s)\,\mathbf L,
\end{equation}
hence $\sigma^2:=\big\|\sum_{i<j}\E[\mathbf Y_{ij}^2]\big\|_2=2p_s(1-p_s)\,\norm{\mathbf L}_2$ and $R:=\max_{i<j}\norm{\mathbf Y_{ij}}_2\le 2$. Applying the matrix Bernstein inequality for sums of independent, self-adjoint, mean-zero matrices~\cite{mhammedi2019pac}:
\begin{equation}
\Prob\!\left\{ \left\|\sum_{i<j}\mathbf Y_{ij}\right\|_2 \ge t \right\}
\le 2n \exp\!\left(-\frac{t^2}{2\sigma^2+\tfrac{2}{3}Rt}\right).
\end{equation}
We then solve for $t$, with probability at least $1-\eta$ for any $\eta\in(0,1)$ to obtain
\begin{equation}
\left\|\sum\mathbf Y_{ij}\right\|_2
\le \sqrt{2\sigma^2\log\frac{2n}{\eta}}+\frac{2}{3}R\log\frac{2n}{\eta}.
\end{equation}
In addition, the bias term satisfies
\begin{equation}
\Big\|\sum_{i<j}\E[\mathbf{X}_{ij}]\Big\|_2
=\;|1-p_s|\,\|\mathbf{L}\|_2.
\label{eq:bias}
\end{equation}
Therefore, by the triangle inequality, with probability at least $1-\eta$,
\begin{align}
\|\Delta\mathbf{L}\|_2
\le &|1-p_s|\,\|\mathbf{L}\|_2\notag\\ 
&+\sqrt{\,4\,p_s(1-p_s)\,\|\mathbf{L}\|_2\,\log(2n/\eta)}\notag\\
&+\frac{4}{3}\,\log(2n/\eta).
\label{eq:deltaL-final}
\end{align}
\label{cor:stability-random}

Finally, under the margin condition $\gamma_{\min}>0$, with probability at least $1-\eta$ for any $\eta\in(0,1)$, the one-layer GCN and the induced classifier satisfy
\begin{align}
\mathrm {Ham}\!\big(\mathbf{y}(G),\mathbf{y}(\widehat{G})\big)
&\le \frac{\sqrt{n}\,C_\sigma\,|h_1|}{\gamma_{\min}}\,
\Big\{|1-p_s|\,\|\mathbf{L}\|_2\notag \\
&+\sqrt{\,4\,p_s(1-p_s)\,\|\mathbf{L}\|_2\,\log(2n/\eta)}\notag\\
&+\tfrac{4}{3}\log(2n/\eta)\Big\} .
\label{eq:hamming-random}
\end{align}
Therefore, Theorem \ref{thm:ps_subsampling_stability} is proved with: 
\begin{align}
  R(\mathbf{y}, \hat{\mathbf{y}})&\le \frac{\mathrm {Ham}\!\big(\mathbf{y}(G),\mathbf{y}(\widehat{G})\big)}{n}\\
  &\le \frac{C_\sigma\,|h_1|}{\sqrt{n} \,\,\,\gamma_{\min}}\,
\Big\{|1-p_s|\,\|\mathbf{L}\|_2\notag \\
&\quad+\sqrt{\,4\,p_s(1-p_s)\,\|\mathbf{L}\|_2\,\log(2n/\eta)}\notag\\
&\quad+\tfrac{4}{3}\log(2n/\eta)\Big\} \triangleq f(p_s, \eta).  
\end{align}
    
\section{Conclusion}

In this work, we study differential privacy in GCNs through the framework of subsampling stability, and formalize a stability-based PTR mechanism (AsampGCN) built on edge-level subsampling and majority-vote aggregation for private node labeling. We derive misclassification-rate bounds that depend explicitly on the subsampling probability $p_s$, and characterize the privacy–utility trade-off via feasible ranges of $p_s$: if $p_s$ is too large, the sufficient condition required by the stability-based DP mechanism may fail; whereas overly small $p_s$ leads to accuracy degradation. Overall, our results provide a rigorous theoretical foundation for understanding subsampling stability in GCNs under DP and for guiding parameter selection in private graph learning.

\bibliographystyle{IEEEbib}
\bibliography{strings,refs}

\end{document}